\documentclass{bmvc2k}

\usepackage{url}
\usepackage{amsmath}
\usepackage{amsthm}
\usepackage{amssymb}
\usepackage{booktabs}
\usepackage{algorithm}
\usepackage{algorithmic}
\usepackage{epsfig}
\usepackage{graphicx}
\usepackage{float}
\usepackage{adjustbox}
\usepackage{color}
\usepackage{comment}
\usepackage{enumitem}
\usepackage{multirow}
\usepackage{subfigure}
\usepackage{subcaption}
\usepackage{xcolor}
\usepackage{multicol}
\newtheorem{theorem}{Theorem}
\newtheorem{proposition}{Proposition}
\usepackage{amsmath,amssymb,amsthm}

\title{Extreme Model Compression with Structured Sparsity at Low Precision}

\addauthor{Dan Liu}{daniel.liu@mail.mcgill.ca}{1,2}
\addauthor{Nikita Dvornik}{
dvornik.nikita@gmail.com}{2}
\addauthor{Xue Liu}{xueliu@cs.mcgill.ca}{1,2,3}

\addinstitution{
 McGill University\\
 Montreal, Canada
}
\addinstitution{
 Palona AI,\\
 Montreal, Canada
}
\addinstitution{
 MBZUAI,\\
 Abu Dhabi, UAE
}

\runninghead{\scriptsize{SLOPE}}{\scriptsize{Extreme Model Compression with Structured Sparsity at Low Precision}}

\begin{document}

\maketitle

\begin{abstract}
Deep neural networks (DNNs) are used in many applications, but their large size and high computational cost make them hard to run on devices with limited resources. Two widely used techniques to address this challenge are weight quantization, which lowers the precision of all weights, and structured sparsity, which removes unimportant weights while retaining the important ones at full precision.
Although both are effective individually, they are typically studied in isolation due to their compounded negative impact on model accuracy when combined.
In this work, we introduce SLOPE (\underline{S}tructured Sparsity at \underline{Lo}w \underline{P}r\underline{e}cision), a unified framework, to effectively combine structured sparsity and low-bit quantization in a principled way.
We show that naïvely combining sparsity and quantization severely harms performance due to the compounded impact of both techniques.
To address this, we propose a training-time regularization strategy that minimizes the discrepancy between full-precision weights and their sparse, quantized counterparts by promoting angular alignment rather than direct matching.
On ResNet-18, SLOPE achieves $\sim20\times$ model size reduction while retaining $\sim$99\% of the original accuracy. It consistently outperforms state-of-the-art quantization and structured sparsity methods across classification, detection, and segmentation tasks on models such as ResNet-18, ViT-Small, and Mask R-CNN.
\end{abstract}

\section{Introduction}
\label{sec:intro}

Deep neural networks (DNNs) are increasingly deployed across a wide range of applications, but their growing size and computational demands present major obstacles for deployment on resource-constrained hardware.
Quantization has emerged as one of the most effective strategies to address these limitations, reducing the bit-width of weights and activations to lower memory usage and accelerate inference~\cite{gholami2021survey}.
By mapping high-precision (e.g., 32-bit floating-point) weights of a trained model to lower-precision representations (e.g., 8-bit or 4-bit), quantization significantly reduces the storage and computational cost of neural networks.
This process often involves a delicate trade-off: while lower precision leads to more compact models and faster execution, it also degrades model accuracy due to reduced representational capacity.
In practice, 4-bit quantization can provide up to 8$\times$ memory reduction compared to 32-bit weights while maintaining competitive performance for many tasks, making it an attractive near-lossless compression method \cite{gholami2021survey}.
However, in many deployment scenarios, particularly those involving large models or strict memory and latency constraints, this level of compression may still fall short \cite{zhu2024llm_quant_survey}.
Pushing the precision even further down may offer additional savings, but at what cost?
Quantizing weights below 4 bits often leads to sharp drops in accuracy, making such extreme quantization levels impractical for most real-world applications\cite{gholami2021survey,oo-icml-nagel22a,ofree-liu23w,esser2019lsq}.
This raises a natural question: can we go beyond quantization to achieve higher compression without incurring severe accuracy loss?

\begin{figure}[hbpt!]
    \centering
    \begin{adjustbox}{max width=\textwidth}
    \includegraphics[width=0.7\linewidth]{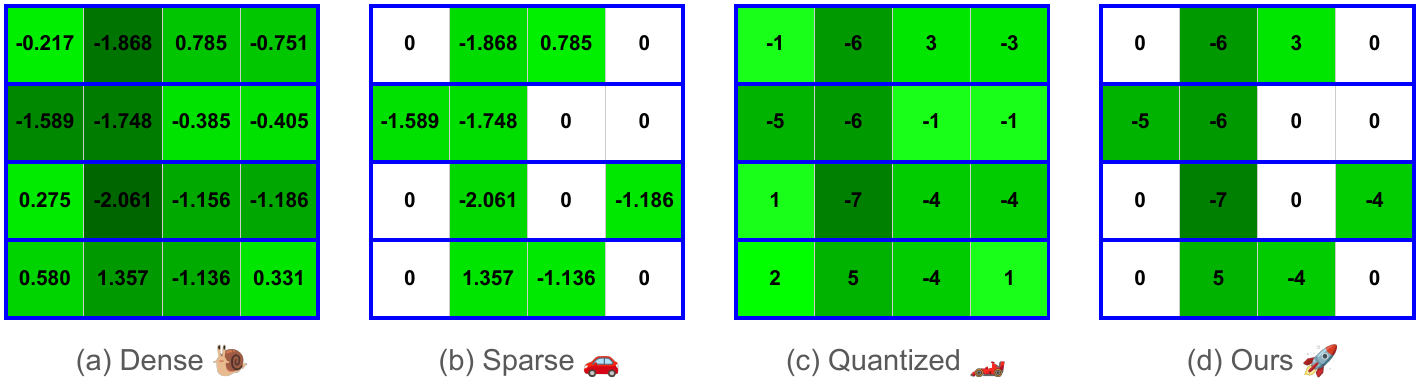}
    \end{adjustbox}
    \vspace{0.5em}
    \caption{\small{Weight matrix representations under different compression settings.  
(a) Dense, full-precision weights offer high accuracy but are computationally expensive.  
(b) Structured 2:4 sparsity (every 4 elements contain 2 non-zeros) in full precision reduces the number of multiplications.
(c) Quantization (4-bit) compresses memory usage.
(d) Structured 2:4 quantization enables much higher inference speedups and compression ratios (See the Appendix). }}
    \label{fig:SLOPE}
    \vspace{-0.5em}
\end{figure}

Another effective approach for model compression is structured N:M sparsity, which enforces that exactly M out of every N consecutive weights are zero (e.g., 2:4 sparsity ensures that in every two out of four consecutive weight elements are non-zero) \cite{bai-2023_nm_nvidia,sparsenm_zhou_2021}.
Beyond reducing the model’s memory footprint, structured sparsity also enables significant computational speedups on modern hardware due to its compatibility with optimized sparse matrix operations~\cite{mishra2021accelerating,frantar2025marlin}.
However, setting the sparsity ratio too high (e.g., 2:16) removes too much information from the original weights and leads to degraded model accuracy, thus limiting the extent of effective compression. 


In this work, we push the limits of model compression and propose to combine structured sparsity with low-bit quantization in a unified framework (see Fig~\ref{fig:SLOPE}-d).
These two techniques reduce model size through fundamentally different mechanisms, sparsity removes weights~\cite{mishra2021accelerating}, while quantization reduces their precision, and are therefore expected to be complementary~\cite{hubara2021accelerated_tmask}.
Yet, in our experiments, we find that applying both simultaneously often leads to a significant drop in accuracy, due to their compounded impact on model capacity.
While finetuning the resulting sparse and quantized model can partially recover performance, a noticeable gap remains compared to the original full-precision model \cite{chmiel2023minimum}.
To address this challenge, we introduce SLOPE (\underline{S}tructured Sparsity at \underline{Lo}w \underline{P}r\underline{e}cision), a novel training framework that integrates N:M structured sparsity with low-bit quantization in a principled way.
At the core of SLOPE is a regularization strategy that minimizes the discrepancy between the original full-precision weights and their sparse, quantized counterparts.
In particular, our regularizer promotes directional alignment between the original and sparse quantized weights, serving as a good prior for high-performance models while allowing for flexibility in sparse quantized weight optimization.
By aligning the original weight vectors with the compressed ones, SLOPE enables the model to retain higher accuracy while enjoying the memory and computational benefits of extreme compression.

We validate SLOPE across a range datasets such as ImageNet \cite{deng2009imagenet} and MS-COCO~\cite{lin2014mscoco}, with ResNet18~\cite{he2016resnet}, DeiT~\cite{deit_touvron2020training} and Mask-RCNN~\cite{he2017maskrcnn} models.
SLOPE consistently outperforms state-of-the-art sparse quantization methods under 2:4 structured sparsity. On ResNet-18, it boosts the accuracy of 4-bit sparse model from 68.36\% to 71.11\%, exceeding the full-precision baseline. On DeiT-small, it achieves 80.7\% Top-1 accuracy, surpassing all existing 2:4 sparse methods. For object detection, SLOPE improves Mask R-CNN box mAP from 37.80 to 40.83 under 4-bit 2:4 sparsity, narrowing the gap to the full-precision baseline (41.0 mAP). These results confirm that SLOPE enables highly compressed models without compromising performance.
To this end, in this work we make the following contributions:
\begin{itemize}
\item We introduce SLOPE, a unified framework that combines structured N:M sparsity and low-bit quantization for extreme model compression.
\vspace{-0.5em}
\item We propose a novel directional regularization strategy that stabilizes training under aggressive compression by aligning compressed and full-precision weight vectors.
\vspace{-0.5em}
\item We demonstrate that SLOPE outperforms state-of-the-art quantization and structured sparsity baselines across different models and computer vision tasks.
\end{itemize}



\section{Related Work}

\subsection{Quantization}

Quantization reduces the precision of model parameters and activations, typically from 32-bit floating point to lower-bit formats such as 4-bit or 2-bit, to decrease memory usage and accelerate inference~\cite{gholami2021survey,esser2019lsq,zhou2016dorefa,rastegari2016xnor,hubara2016binarized}.  
A central goal is to ensure that the quantized weights \(\widehat{\mathbf{W}}\) closely approximate the full-precision weights \(\mathbf{W}\), thereby preserving model performance.  

Many works aim to minimize this discrepancy.  
For instance, \cite{li2016twn,oo-icml-nagel22a,ofree-liu23w} directly minimize the \(L_2\) distance between \(\widehat{\mathbf{W}}\) and \(\mathbf{W}\).  
Others \cite{bin_tiantianhan2021improving,kim2020position,liu2023hyperspherical} promote quantization-aware training by encouraging \(\mathbf{W}\) to lie near quantization bin centers.  
Lin et al.~\cite{lin2020rotated} further enhance quantization by learning a set of rotation matrices that iteratively align \(\mathbf{W}\) and \(\widehat{\mathbf{W}}\), yielding significant improvements in binary settings. More quantization works are discussed in the work of \citet{gholami2021survey}.

\subsection{Structured N:M Sparsity}

Structured N:M sparsity is a specific and increasingly popular form of structured sparsity (i.e., the type of sparsity where the weights are sparsified according to a specific pattern).  
It enforces a fine-grained constraint in which only \(N\) weights are retained out of every \(M\) consecutive elements.  
This approach strikes a balance between the flexibility of unstructured sparsity and the hardware efficiency of coarse-grained patterns.  
Notably, NVIDIA GPUs support 2:4 sparsity at inference time, enabling practical acceleration~\cite{bai-2023_nm_nvidia}.  
With careful fine-tuning, such structured sparsity can retain performance comparable to dense models.  

Several techniques have been proposed to train N:M sparse models effectively.  
Zhou et al.~\cite{sparsenm_zhou_2021} introduce SR-STE, a straight-through estimator adapted for N:M constraints, enabling sparse models to be trained from scratch with minimal accuracy loss and achieving up to 2\(\times\) speedup on NVIDIA A100 GPUs.  
LBC~\cite{zhang2022learning} addresses the combinatorial nature of N:M sparsity via a divide-and-conquer approach that assigns learnable scores to weight subsets, enabling efficient mask learning.  
STEP~\cite{lu2023step} proposes an Adam-aware method for mask learning, consisting of a preconditioning phase for estimating reliable gradient variances followed by a sparsity-inducing optimization phase.  

Beyond weights, Chmiel et al.~\cite{chmiel2023minimum} apply N:M sparsity to gradients during training, introducing a minimum-variance unbiased estimator (MVUE) that supports 1:2 or 2:4 sparse gradients, reducing computation without harming convergence.  
Finally, S-STE~\cite{hu2024sste} proposes a continuous pruning framework for pretraining sparse models.  
It uses a projection-based pruning function and fixed rescaling of sparse weights to produce efficient 2:4 sparse models that closely match the performance of their dense counterparts.  

Recent efforts have explored sparse quantization, which jointly applies sparsity and quantization to maximize model compression. 
Harma et al.~\cite{harma2024effective} highlight the importance of aligning sparse structures with quantization to mitigate compounded degradation and propose a unified framework that jointly optimizes both constraints to maintain accuracy.  
Similarly, Guo et al.~\cite{guo2024_JSQ} identify key challenges in sparse quantization, such as the tendency of sparsification to preserve outliers that complicate downstream quantization.  
In contrast, SLOPE focuses on aligning the full-precision weights with their sparse quantized counterparts, yielding complementary improvements.

\section{Preliminaries}
In this section we introduce formal definitions of weight quantization, structured N:M sparsity, and cover common ways to measure discrepancy between original and compressed weights of a neural network. In this work,
a linear layer is defined as: 
$
    \mathbf{y}={\mathbf{W}}^\top\mathbf{x},
$
where $\mathbf{x}\in\mathbb{R}^{m\times{1}}$ denotes the input vector and ${\mathbf{W}} \in \mathbb{R}^{m\times{n}}$ denotes the weight matrix with $i = 1, . . . , n$. $\mathbf{{w}}_i\in{\mathbf{W}}$ denotes the $i$-th vector of ${\mathbf{W}}$. $\mathbf{y}\in\mathbb{R}^{n\times{1}}$ is the layer output. 

\paragraph{Quantization}
In this paper, by quantization we mean rounding float values to their nearest lower-precision counterpart.
This corresponds to the LSQ \cite{esser2019lsq} formalism:
\begin{equation}
\label{eq_quant}
\scalebox{0.7}{$
\widehat{\mathbf{W}} = q(\mathbf{W}; s, b) 
= s \Big[\operatorname{clamp}\Big(\Big\lfloor \frac{\mathbf{W}}{s} \Big\rceil \,;\, -\text{Q}_N,\;\text{Q}_P\Big)\Big],
$}
\end{equation}
where \(\mathbf{W}\) and \(\widehat{\mathbf{W}}\) denote the full-precision and quantized weights, respectively. \(s\) is a learnable parameter for scaling weights, \(b\) represents the bit-width, $\lfloor\rceil$ is the rounding operator. For {activation quantization}, the quantization range is typically defined as 
$
Q_N = 0, \quad Q_P = 2^{b-1},
$
which corresponds to an {asymmetric unsigned mapping} suitable for non-negative activations (e.g., ReLU). 
In contrast, for {weight quantization}, the range is defined as 
$
Q_N = -2^{b-1}, \quad Q_P = 2^{b-1} - 1,
$
which corresponds to a {symmetric signed mapping} centered around zero, reflecting the approximately zero-mean distribution of weights.

\paragraph{Structured N:M Sparsity} 
Structured N:M sparsity keeps N non-zero and prunes M-N weights to zero in every M consecutive elements (e.g., N = 2, M = 4; Fig.~\ref{fig:SLOPE}b). Let \(\mathsf{w}_{N:M} \subset \mathbf{W}\) be a block of M consecutive elements in \(\mathbf{W}\), and let \(\widetilde{\mathsf{w}}_{N:M} \subset \widetilde{\mathbf{W}}\) be the corresponding block in \(\widetilde{\mathbf{W}}\). Then each element \(w_i\in \mathsf{w}_{N:M}\) is mapped to:

\begin{equation}
\label{equ:2}
\scalebox{0.7}{$
   \widetilde{w}_i = S(w_i; N, M)=
       \begin{cases}
       w_i, &\text{if } |w_i| \ge \xi,\\
       0,   &\text{if } |w_i| < \xi,
       \end{cases}
   \quad \text{for } i =1,2,\ldots,M
$}
\end{equation}

\noindent where \(\xi\) is the N-th largest absolute value in the set \(\{ |w_1|, |w_2|, \dots, |w_M| \}\).
Essentially, the N:M sparsification operation zeroes out \((M - N)\) smallest elements in each M-element group, while keeping the remaining N values intact.

\paragraph{Weight Deviation Measures}
There are multiple ways to measure the deviation of the compressed weights with respect to the original ones.
Some common measures include L$_2$, L$_1$ and sometimes cosine distance, which are typically used to to get the overall magnitude (or direction) of change.
Another application-specific measure is Signal-to-Quantization-Noise Ratio (SQNR). 
It captures how much \(\widehat{\mathbf{W}}\) deviates from \(\mathbf{W}\) in a way that matters most in quantization.  
Formal definition of QSNR is as follows:
   \begin{equation}
   \label{eq:sqnr}
   \scalebox{0.7}{$
   \mathrm{SQNR}
   \;=\;
   10 \log_{10}\!\Bigl(
     \frac{
       \frac{1}{n}\sum_{i=1}^n \|\mathbf{w}_i\|^2
     }{
       \frac{1}{n}\sum_{i=1}^n \|\mathbf{w}_i - \hat{\mathbf{w}}_i\|^2
     }
   \Bigr).
   $}
   \end{equation}
It measures the ratio between the maximum nominal signal strength and the quantization error. Higher SQNR means lower quantization error {\cite{kuzmin2023pruning,lin2016fixed,torch_sqnr-2024}.

\section{Our Approach}
In this section we introduce SLOPE, our approach for combining structured sparsity with quantization.
We elaborate on the design of our regularization function and detail the training and inference procedure of SLOPE.
Yet, to motivate the need for our approach, we start with a simple experiment where quantization and sparsity are combined naively.
\begin{table}[h]
\centering
\small
\begin{adjustbox}{max width=0.8\textwidth}

\begin{tabular}{l|cc|cc|cc}
\toprule
\multirow{2}{*}{Metric} & \multicolumn{2}{c|}{Full-precision} & \multicolumn{2}{c|}{4-bit Quant} & \multicolumn{2}{c}{2-bit Quant} \\
                        & w/o Sparse & + 2:4 Sparse           & w/o Sparse & + 2:4 Sparse         & w/o Sparse & + 2:4 Sparse \\
\midrule
Test Acc. (\%) $\uparrow$ & 69.76 & \textbf{69.96} & 71.10 & 68.36 & 67.60 & 64.47 \\

Cosine $\uparrow$       & NA         & \textbf{0.975 $\pm$ 0.01} & 0.948 $\pm$ 0.01 & 0.919 $\pm$ 0.01 & 0.878 $\pm$ 0.03 & 0.831 $\pm$ 0.03 \\
SQNR (dB) $\uparrow$    & NA         & \textbf{14.74 $\pm$ 0.98} & 10.51 $\pm$ 1.01 & 8.43 $\pm$ 1.04  & 7.54 $\pm$ 1.08  & 4.84 $\pm$ 1.13 \\
\bottomrule
\end{tabular}
\end{adjustbox}
\vspace{0.3cm}
\caption{\small{Impact of combining structured 2:4 sparsity and quantization on top-1 ImageNet accuracy of ResNet‐18. Cosine and SQNR depict weight deviations from the full-precision weights.}}
\label{tab_drifts}
\vspace{-0.7cm}
\end{table}


        
    

\subsection{Our Motivation: Weight Deviation in Sparse Quantization}
Structured sparsity and weight quantization achieve compression through fundamentally different mechanisms: quantization reduces the numerical precision of weights, while sparsity retains only a subset of weights. A natural idea is to combine these techniques: first applying sparsity, then quantizing the non-zero weights, followed by light fine-tuning to adapt the surviving values while preserving the sparsity pattern. This approach, known as \emph{sparse quantization}~\cite{mishra2021accelerating}, intuitively leverages the strengths of both methods. However, when applied to models such as ResNet-18, we observe a significant drop in accuracy compared to using sparsity or quantization alone, particularly at lower bit-widths (see Table~\ref{tab_drifts}). We hypothesize that this degradation arises from the compounded distortion induced by both compression techniques. This is supported by angular deviation and SQNR analyses, which show markedly larger weight discrepancies when sparsity and quantization are combined. These findings highlight the importance of controlling both magnitude and directional deviations of weight vectors to preserve accuracy in sparse low-bit quantization, and motivate the need for a principled solution.



\subsection{SLOPE: Structured Sparsity at Low Precision}
To address the performance drop from combining structured sparsity and quantization, we add a regularizer during fine-tuning that explicitly encourages the full-precision weights $\mathbf{W}$ to stay close to the sparse quantized weights $\widehat{\mathbf{W}}$. The optimization process is formulated as:
\begin{equation}
\label{eq:loss}
\scalebox{0.7}{$
\min_{{\mathbf{W}}} J(\mathbf{{W}}) = L(\widehat{\mathbf{W}}) + \lambda L_{reg}(\mathbf{W}, \widehat{\mathbf{W}}),
$}
\end{equation}
where \(
\widehat{\mathbf W}=q\!\bigl(S(\mathbf W;N,M);s,b\bigr)
\), $L()$ denotes a standard objective function,  and $\lambda$ is the regularizer strengths empirically set based on the loss value of $L()$ to make sure $L_{reg}$ and $L$ are at the same scale. The $L_{reg}$ is defined by: 
\begin{equation} 
    \label{eq:ltr}
    \scalebox{0.7}{$
    L_{reg}(\mathbf{W},\mathbf{\widehat{W}})=\frac{1}{n}\sum_{i=1}^{n}\left(1-\texttt{cos}(\mathbf{w}_i,\mathbf{\hat{w}}_i)\right) ,
    $}
\end{equation}
where $\texttt{cos}()$ denotes calculating the cosine similarity. Minimizing $L_{reg}$ reduces the cosine distance (angular deviation) between $\mathbf{w}_j$ and $\mathbf{\hat{w}}_j$. The closer the distance between $\mathbf{w}_j$ and $\mathbf{\hat{w}}_j$ is, the smaller the deviation $\theta$ is during the sparse quantization.


An ideal regularizer $L_{\text{reg}}$ should, in principle, reduce the discrepancy between $\mathbf{W}$ and $\widehat{\mathbf{W}}$ to zero given sufficient weight precision, making $L_1$ or $L_2$ penalties appealing choices~\cite{bin_tiantianhan2021improving}. However, in the presence of structured sparse quantization, perfect alignment is no longer achievable. As shown in our results, the imposed low-precision sparse pattern significantly alters the weight direction, and standard $L_1$/$L_2$ penalties fail to effectively constrain angular deviation. As demonstrated in our ablation results (Table~\ref{tab:ablation_quant}), minimizing $L_2(\mathbf{W}, \widehat{\mathbf{W}})$ becomes ineffective under low-precision sparsity.

To this end, we propose SLOPE, which relaxes the regularization objective and instead optimize an upper bound that does not impose per-element constraints on the weight matrix.
In the structured sparsity setting, we can show that the $L_2$ distance is upper bounded by:
\begin{equation}
\label{eq:angle_app}
\scalebox{0.7}{$
\|\mathbf{w} - \hat{\mathbf{w}}\|^2_2 \leq \underbrace{2\|\mathbf{w}\|^2_2(1 - \cos\theta)}_{\text{Upper bound}}.
$}
\end{equation}
While the upper bound scales with the weight norm, its minimization is driven by the angular term only.
This regularizer promotes directional alignment without forcing sub-optimal individual weight values.
SLOPE leverages this property to recover most of the lost accuracy under compression, making angular regularization an effective strategy for sparse quantization. (See Appendix for derivation and geometric interpretation.)



\paragraph{SLOPE Implementation: Going Forward and Backwards}
Here we detail how the forward and backward passes are performed, in the presence of quantization and sparsification.
During the forward pass, following \citet{harma2024effective}, we first sparsify the full-precision
weights and then quantize them,
\(
\widehat{\mathbf W}=q\!\bigl(S(\mathbf W;N,M);s,b\bigr).
\)The sparse quantized weights are optimized with the straight-through
estimator (STE) \cite{bengio2013estimating,esser2019lsq,lcq_yamamoto2021learnable,sparsenm_zhou_2021,oo-icml-nagel22a}:
\begin{equation}
\label{eq:ste}
\scalebox{0.7}{$
\left\{
\begin{aligned}
&\widehat{\mathbf{W}}_t = q\big(S(\mathbf{W}_t;N, M); s, b\big),\\
&\mathbf{W}_{t+1} = \mathbf{W}_t - \frac{\partial L(\widehat{\mathbf{W}}_t)}{\partial \widehat{\mathbf{W}}_t},
\end{aligned}
\right.
$}
\end{equation}
where \(t\) represents the training iteration and \(\frac{\partial L(\widehat{\mathbf{W}}_t)}{\partial \widehat{\mathbf{W}}_t}\) is the approximated gradient.



\section{Experiments}
In this section, we conduct experiments across a diverse set of models and tasks to demonstrate the generalization ability of our method. We evaluate ResNet-18~\cite{he2016resnet} and DeiT-small~\cite{deit_touvron2020training} on ImageNet~\cite{deng2009imagenet} for classification, and Mask R-CNN~\cite{wu2019detectron2} on MS-COCO~\cite{lin2014mscoco} for detection and segmentation. Our study covers both weight-only and weight-activation quantization, combined with 2:4 structured sparsity. We also explore extreme sparse configurations such as 2:8 and 2:16. Evaluation metrics include Top-1 accuracy on ImageNet and mean Average Precision (mAP)~\cite{wu2019detectron2} on MS-COCO. Additional training details (e.g., overhead and training time) and sparse-only results are provided in the Appendix.

\subsection{Image Classification}

We evaluate SLOPE against SoTA compression methods that use quantization and structured sparsity (i.e., ASP~\cite{mishra2021accelerating}, SR-STE~\cite{sparsenm_zhou_2021}, and MVUE~\cite{chmiel2023minimum}), or quantization alone (i.e., Quant~\cite{bin_tiantianhan2021improving}).
We can see that quantization generally degrades performance when combined with structured sparsity.
However, for {ResNet18}, {SLOPE} achieves 71.11\% Top-1 accuracy under 4-bit quantization, exceeding the full-precision baseline (69.76\%), while using memory close to 2-bit precision. 
While it may seem surprising, sparse models are generally known to outperform their dense counterparts thanks to the sparsity-induced regularization, yet, we are the first to show this result under quantization!
For {ResNet50}, {SLOPE} also consistently outperforms other methods. It reaches 75.93\% Top-1 accuracy in the 4-bit setting and 72.30\% with 2-bit quantization, both of which are significantly better than other baselines. Notably, even under aggressive 2-bit compression, SLOPE maintains competitive accuracy compared to 4-bit results of prior work.
This demonstrates SLOPE’s strong generalization ability and robustness under extreme quantization.

\begin{table}[ht]
    \centering
    \begin{adjustbox}{max width=0.7\textwidth}
    \begin{tabular}{l|cc|cc|cc}
    \toprule
    \multirow{2}{*}{\textbf{Method}} 
    & \multicolumn{2}{c|}{\textbf{FP}} 
    & \multicolumn{2}{c|}{\textbf{4-bit}} 
    & \multicolumn{2}{c}{\textbf{2-bit}} \\
    \cmidrule(lr){2-3} \cmidrule(lr){4-5} \cmidrule(lr){6-7}
    & ResNet-18 & ResNet-50 
    & ResNet-18 & ResNet-50 
    & ResNet-18 & ResNet-50 \\
    \midrule
    Quant              & 69.76         & 76.13         & 71.10         & 76.70         & 67.60         & 73.70 \\
    ASP                & 69.90$^\dagger$ & 76.80$^\dagger$ & 68.36         & 74.70         & 64.47         & 71.20 \\
    SR-STE             & 71.20$^\dagger$ & 77.00$^\dagger$ & 69.23         & -             & 63.71         & -     \\
    MVUE               & 70.6             & 77.12             & 67.22         & -             & -             & -     \\
    \textbf{SLOPE}     & \textbf{71.23}$^\dagger$ & \textbf{77.24}$^\dagger$ & \textbf{71.11} & \textbf{75.93} & {67.59} & {72.34} \\
    \bottomrule
    \end{tabular}
    \end{adjustbox}
    \vspace{0.3cm}
    \caption{\small{Comparison of different structured 2:4 sparse quantization methods on ResNet models under various bit-widths. ${\dag}$ denotes training-from-scratch. ``Quant'' denotes quantization only.}}
    \label{tab:resnet_slope_main}
    \vspace{-0.3cm}
\end{table}

In Fig.~\ref{fig:slope_vs_all}, we plot SLOPE's performance in comparison to other baselines under different quantization bit-widths and compression levels.
It is clear that SLOPE consistently outperforms (or performs on par) all sparsity- and/or quantization methods across 8-bit, 4-bit, and 2-bit settings, with particularly large gains at lower precisions. Fig.~\ref{fig:slope_vs_all} shows top-1 accuracy and compression savings against FP32, where SLOPE achieves the best accuracy under aggressive compression (up to 93.75\%).

\vspace{-0.8em}

\paragraph{DeiT-small}

For DeiT-small on ImageNet, we evaluate 2:4 structured sparsity under 2-bit and 4-bit quantization. Table~\ref{tab_deit_nbit} shows results with quantized activations (A16/A4/A2), denoted as A16/W4, A4/W4, and A2/W2. A16 denotes activation quantization with bfloat16 precision. SLOPE consistently outperforms the baseline, especially in low-bit settings. Under the challenging 2-,4- and 8-bit configurations, SLOPE achieves consistently outperforms the baseline, by up to 1.8\%.

\begin{figure}[ht]
    \centering
    \begin{minipage}{0.48\textwidth}
        \centering
        \begin{adjustbox}{max width=\textwidth}
        \begin{tabular}{lccccc}
        \toprule
        \textbf{Method} & \textbf{A16/W4} & \textbf{A16/W2} & \textbf{A8/W8} & \textbf{A4/W4} & \textbf{A2/W2} \\
        \midrule
        {Dense}       & 80.87 & 77.34 & 79.56 & 80.33 & 75.72 \\
        \midrule
        {ASP}         & 79.28 & 75.76 & 78.75 & 77.81 & 60.77 \\
        {SLOPE}       & \textbf{80.41} & \textbf{76.81} & \textbf{80.59} & \textbf{78.15} & \textbf{61.32} \\
        \bottomrule
        \end{tabular}
        \end{adjustbox}
        \vspace{0.3cm}
        \captionof{table}{\small{Results of sparse 2:4 quantization on DeiT-small with ImageNet.}}
        \label{tab_deit_nbit}
    \end{minipage}
    \hfill
    \begin{minipage}{0.48\textwidth}
        \centering
        \includegraphics[width=\textwidth]{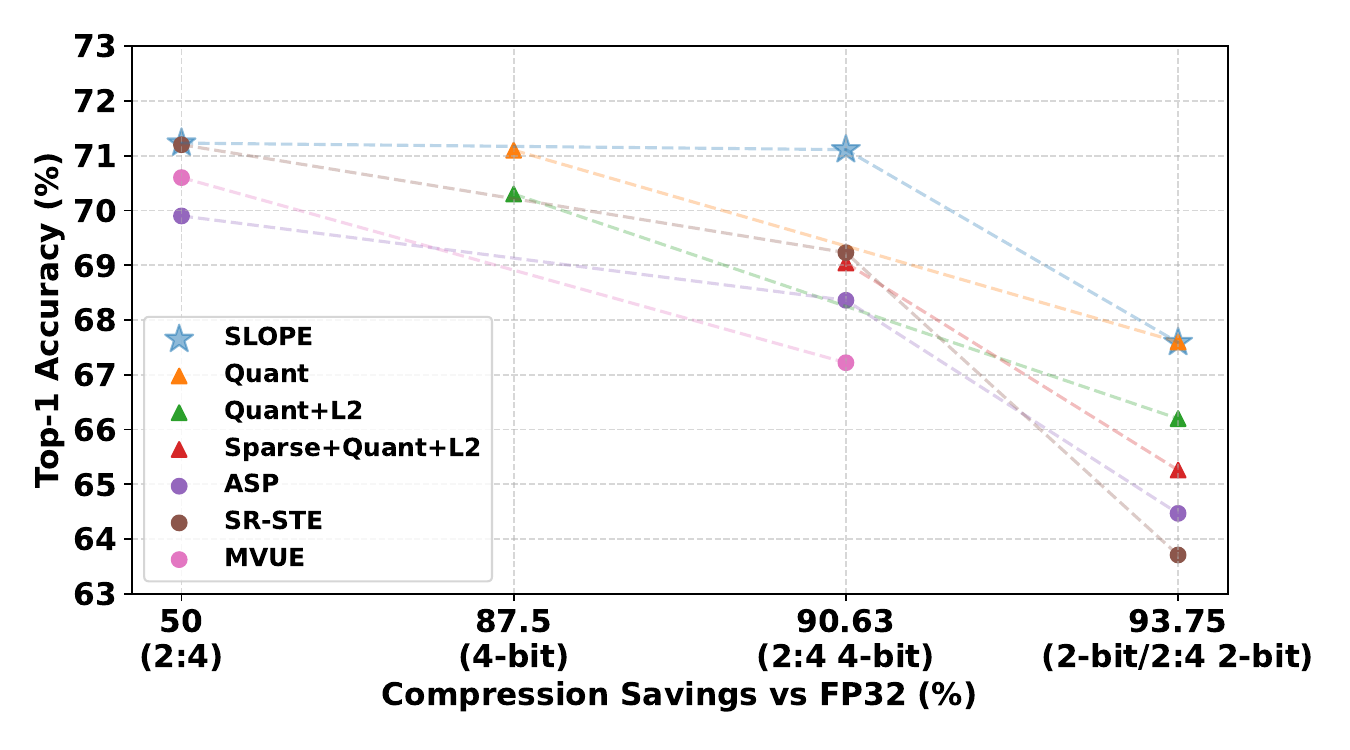} 
        \captionof{figure}{\small{Accuracy vs. compression ratio on ResNet-18 models.}}
        \label{fig:slope_vs_all}
    \end{minipage}
    \vspace{-0.3cm}
\end{figure}

Notably, SLOPE shows clear advantages under similar memory budgets: sparse A16/W4 outperforms dense A16/W2 (80.41\% vs. 77.34\%), and sparse A8/W8 exceeds dense A4/W4 (80.59\% vs. 78.15\%) (Table~\ref{tab:2to4_savings_fp32}). In addition, SLOPE's A4/W4 accuracy (78.15\%) exceeds that of dense A2/W2 (75.72\%). These results highlight SLOPE's superior balance between compression and accuracy.

\subsection{Object Detection and Segmentation}

In this section we measure the effect of SLOPE's compression on object detection and segmentation, and compare it to other 2:4 sparse quantization methods on the COCO dataset.
In Table 4, besides reporting the numbers achieved with finetuning, we also show the performance when training the models from scratch.
In both settings, SLOPE achieves significantly better performance than the baseline~\cite{wu2019detectron2}, even under structured 2:4 sparsity. Specifically, under the A32/W4 setting, SLOPE yields 40.83 box mAP and 37.13 mask mAP, outperforming the baseline (which applies quantization only) by +3.03 and +2.21, respectively.
A similar trend holds for the more challenging A4/W4 configuration, where SLOPE achieves 38.79/35.28 compared to the baseline’s 29.64/27.59, highlighting the robustness of SLOPE even under both quantization and sparsity constraints.


\begin{table}[h]
\centering

\begin{minipage}[t]{0.48\linewidth}
\centering
\begin{adjustbox}{width=0.8\textwidth}
\begin{tabular}{llcc}
\toprule
\textbf{Method} & \textbf{Bits}    & \textbf{Box$_{mAP}$} & \textbf{Mask$_{mAP}$} \\
\midrule
Baseline & & 41.0 & 37.2 \\
\midrule
Dense & W4 & 37.80 & 34.92 \\
SLOPE     & W4 & \textbf{40.83} & \textbf{37.13} \\
\midrule
Dense & A4/W4  & 29.64 & 27.59 \\
SLOPE     & A4/W4  & \textbf{38.79} & \textbf{35.28} \\
\bottomrule
\end{tabular}
\end{adjustbox}
\end{minipage}
\hfill
\begin{minipage}[t]{0.48\linewidth}
\centering
\begin{adjustbox}{width=0.8\textwidth}
\begin{tabular}{llcc}
\toprule
\textbf{Method} & \textbf{Bits}    & \textbf{Box$_{mAP}$} & \textbf{Mask$_{mAP}$} \\
\midrule
Baseline & & 41.0 & 37.2 \\
\midrule
SR-STE   & FP.& 39.0 & 35.3 \\
LBC      & FP.& 39.3 & 35.4 \\
SLOPE     & W4& \textbf{39.3} & \textbf{36.2} \\
\bottomrule
\end{tabular}
\end{adjustbox}
\end{minipage}
\label{table:coco_results}
\vspace{0.3cm}
\caption{\small{Results of object detection with structured 2:4 sparsity on Mask R-CNN models. Left: finetuning from pre-trained models. Right: training-from-scratch models. ``Dense'' denotes without applying 2:4 sparsity.
}}
\vspace{-2em}
\end{table}

\subsection{Analysis}
In this section, we perform additional analysis of SLOPE, studying the effect of weight discrepancy on the performance, and perform the ablation study on the regularization term.
\subsubsection{The Effect of Weight Discrepancy on SLOPE's performance}

Table~\ref{tab_drifts_snmq} demonstrates that reducing weight discrepancy with \textsc{SLOPE} is crucial for achieving strong performance under high compression rates.
Under the 4-bit 2:4 setting, SLOPE improves accuracy from 68.36\% to 71.11\%, surpassing even the full-precision baseline. In the more extreme 2-bit 2:4 case, it improves from 64.47\% to 67.59\%. These gains are supported by alignment metrics: cosine similarity increases from 0.919 to 0.953 (4-bit) and from 0.831 to 0.922 (2-bit); SQNR improves from 8.43\,dB to 11.59\,dB (4-bit) and from 4.84\,dB to 8.92\,dB (2-bit). These results demonstrate that SLOPE effectively reduces both directional and magnitude deviations introduced by sparse low-bit quantization.

\begin{figure}[htbp]
    \centering

    \begin{minipage}[c]{0.48\textwidth}
        \centering
        \begin{adjustbox}{max width=\textwidth}
        \begin{tabular}{lcccc}
        \toprule
        \textbf{Settings} & \textbf{Acc.} & \textbf{Cos} $\pm$ std. & \textbf{SQNR} $\pm$ std. & \textbf{Compression} \\
        \midrule
        2:4 & 69.96 & 0.975 $\pm$ 0.01 & 14.74 $\pm$ 0.98 & 50\% \\
        \midrule
        4-bit & 71.10 & 0.948 $\pm$ 0.01 & 10.51 $\pm$ 1.01 & 87.5\% \\
        2:4, 4-bit & 68.36 & 0.919 $\pm$ 0.01 & 8.43 $\pm$ 1.04 & 90.63\% \\
        2:4, 4-bit, \textsc{SLOPE} & \textbf{71.11} & 0.953 $\pm$ 0.01 & 11.59 $\pm$ 1.09 & 90.63\% \\
        \midrule
        2-bit & 67.60 & 0.878 $\pm$ 0.03 & 7.54 $\pm$ 1.08 & 93.75\% \\
        2:4, 2-bit & 64.47 & 0.831 $\pm$ 0.03 & 4.84 $\pm$ 1.13 & 93.75\% \\
        2:4, 2-bit, \textsc{SLOPE} & \textbf{67.59} & 0.922 $\pm$ 0.02 & 8.92 $\pm$ 1.04 & 93.75\% \\
        \bottomrule
        \end{tabular}
        \end{adjustbox}
        \vspace{0.3cm}
        \captionof{table}{\small{ResNet-18 performance across compression settings with and without \textsc{SLOPE}.}}
        \label{tab_drifts_snmq}
    \end{minipage}
    \hfill
    \begin{minipage}[c]{0.48\textwidth}
        \centering
        \includegraphics[width=\textwidth]{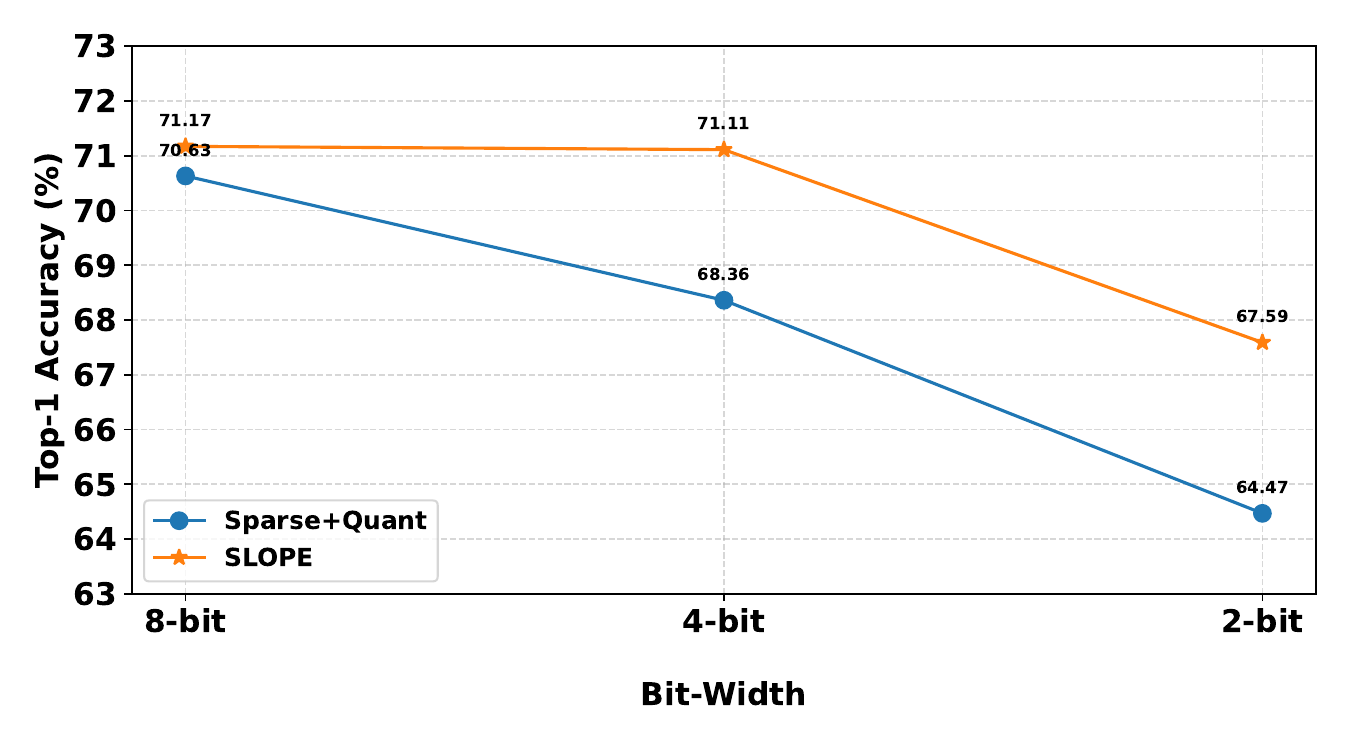}
        \captionof{figure}{\small{Accuracy of 2:4 Sparse+Quant and \textsc{SLOPE} under varying bit-widths on ResNet-18.}}
        \label{fig:left-slope-plot}
    \end{minipage}
    \vspace{-0.5cm}
\end{figure}

\subsubsection{Ablation Study}

Here we study the role of our cosine weight regularizer, and compare it to \(L_2\) penalty on weight discrepancy.
We compare SLOPE to the baseline without regularization under varying sparsity patterns (2:4, 2:8, 2:16) and bit-widths (Table~\ref{tab_ablation_nm_nbit}, Table~\ref{tab:ablation_quant}) on ResNet-18 with ImageNet. The results indicate that higher sparsity ratios (e.g., 2:16) which can offer greater computational savings, lead to a more pronounced drop in accuracy. However, the proposed SLOPE method consistently mitigates this degradation, particularly in low-bit quantization settings, demonstrating its ability to preserve accuracy across varying sparsity levels (Table~\ref{tab_ablation_nm_nbit}, Fig.~\ref{fig:left-slope-plot}). 
These improvements highlight the importance of angular (and not absolute) weight alignment for maintaining representational fidelity \cite{LiuNIPS18energy} and justify our $\cos$ regularizer.
\begin{table}[hbpt!]
\centering
\small
\begin{adjustbox}{width=0.7\textwidth}

\begin{tabular}{l *{2}{c}|*{3}{c}|*{3}{c}|*{3}{c}}
\toprule
 \multicolumn{3}{c}{FP} |& \multicolumn{3}{c}{A8/W8} |& \multicolumn{3}{c}{A4/W4} |& \multicolumn{3}{c}{A2/W2} \\
\textbf{N:M}& \textbf{Baseline} & \textbf{SLOPE} & \textbf{Baseline} & \textbf{$L_2$} & \textbf{SLOPE} & \textbf{Baseline} & \textbf{$L_2$} & \textbf{SLOPE} & \textbf{Baseline} & \textbf{$L_2$} & \textbf{SLOPE} \\
\midrule
2:4  & 70.70 & 71.25 & 70.63 & 70.42 & 71.17 & 68.36 & 69.04 & 71.11 & 64.47&65.26&67.59 \\
2:8  & 69.62 & 70.36 & 69.71 & 69.33 & 70.11 & 66.83 & 67.53 & \textbf{68.77} & 59.76&60.65&61.87 \\
2:16 & 66.73 & 67.94 & 66.85 & 66.81 & 67.93 & 65.38 & 66.07 & 66.59 & -&-&- \\
\bottomrule
\end{tabular}
\end{adjustbox}
\vspace{0.3cm}
\caption{\small{Different settings of sparse N:M n-bit quantization on ResNet-18. ``-'' denotes not converged models. ``Baseline'' does not use any additional loss term. \textbf{The 4-bit 2:8 quantization scheme achieves a $19.7\times$ compression ratio while preserving nearly 99\% of the original accuracy (68.77 vs. 69.76).} Please find the compression ratio in Appendix. }}
\vspace{-2em}
\label{tab_ablation_nm_nbit}

\end{table}
\begin{table}[hbpt!]
\centering
\small
 \begin{adjustbox}{max width=0.7\textwidth}
\begin{tabular}{lcccccc}
\toprule
\textbf{Bits} & \textbf{Quant} & \textbf{Quant+$L_2$}~\cite{bin_tiantianhan2021improving} & \textbf{Sparse$_{2:4}$+Quant} & \textbf{Sparse$_{2:4}$+Quant+$L_2$} & \textbf{SLOPE$_{2:4}$} \\
\midrule
2-bit & 67.60 & 66.20 & 64.47 & 65.26 & \textbf{67.59} \\
4-bit & 71.10 & 70.30 & 68.36 & 69.04 & \textbf{71.11} \\
\bottomrule
\end{tabular}
\end{adjustbox}
\vspace{0.3cm}
\caption{\small{Comparison of different quantization and regularization strategies under 2-bit and 4-bit settings with ResNet-18 models.}}
\label{tab:ablation_quant}
\end{table}

Under both 4-bit and 2-bit configurations (Table~\ref{tab:ablation_quant}), SLOPE significantly outperforms \( L_2 \), indicating that aligning weight directions (as SLOPE encourages) is more effective than minimizing Euclidean distance alone. Moreover, SLOPE surpasses the baseline in all settings, validating the effectiveness of our loss design for sparse low-bit quantization.

\section{Conclusion}
A key contribution of this work is identifying the significant performance degradation when combining weight quantization and structured sparsity for model compression.
We then show how identifying the root cause: high weight discrepancy between compressed and original weights, and fixing it with a novel regularizer, can recover most of the lost performance.
 Through extensive experiments on diverse models and datasets, such as ResNet-18, ViT-small, and Mask R-CNN, we demonstrate that our proposed method (SLOPE) significantly enhances the performance of quantized sparse models. 
 We show that SLOPE is capable of aggressive compression at low bit rates with induced 2:4 structured sparsity, while maintaining most of the original performance.
 We believe that this research uncovers the fundamental mechanisms of how sparsity and quantization interact, and proposes new tools and methods for the research community.

\bibliography{egbib}
\newpage
\section*{Appendix}
\subsection*{Sparse-only Results}
Table~\ref{tab_sparse_deit} compares our proposed $L_{reg}$ with other leading sparse 2:4 training methods such as Bi-Mask~\cite{zhang2023bi}, SR-STE\cite{sparsenm_zhou_2021}, T-Mask\cite{hubara2021accelerated_tmask}, LBC~\cite{zhang2022learning}, and S-STE~\cite{hu2024sste}. Our method achieves the best Top-1 accuracy of 80.7\%, outperforming all baselines.

\begin{table}[ht]
\centering
\begin{adjustbox}{max width=0.5\textwidth}
\begin{tabular}{ccccccc}
\toprule
 \textbf{ASP}\cite{mishra2021accelerating} & \textbf{Bi-Mask} & \textbf{SR-STE} & \textbf{T-Mask} & \textbf{LBC} & \textbf{S-STE} & \textbf{$L_{reg}$} \\
\midrule
 79.9 & 77.6 & 79.6 & 71.5 & 78.0 & 78.5 & \textbf{80.7} \\
\bottomrule
\end{tabular}
\end{adjustbox}
    \vspace{0.5em}
    \caption{Results of 2:4 sparse-only on DeiT-small models with ImageNet dataset.}
    \label{tab_sparse_deit}
\end{table}

\subsection*{Efficiency of $n$-bit Structured 2:4 Sparsity}

Implementing $n$-bit structured $2$:$4$ sparsity yields significant memory savings and computational speedups~\cite{mishra2021accelerating}. With 8-bit quantization, this yields up to $2\times$ speedup with negligible accuracy loss~\cite{mishra2021accelerating,hu2024sste,sparsenm_zhou_2021}, while 4-bit quantization achieves up to $4\times$ speedup~\cite{frantar2025marlin}. Table~\ref{tab:2to4_savings_fp32} summarizes the storage compression ratio relative to 32-bit dense weights.

\begin{table}[h]
\centering
\begin{adjustbox}{max width=0.8\textwidth}

\begin{tabular}{clccc}
\toprule
\textbf{Bitwidth ($n$)} & \textbf{2:4 Sparse (bits)} & \textbf{Savings vs FP32} & \textbf{Formula}&\textbf{Compression Ratio} \\
\midrule
8-bit   & $2 \times 8 + 4 = 20$   & 84.38\% & $\frac{128 - 20}{128}$&6.4$\times$ \\
4-bit   & $2 \times 4 + 4 = 12$   & 90.63\% & $\frac{128 - 12}{128}$&10.7$\times$ \\
2-bit  & $2 \times 2 + 4 = 8$    & 93.75\% & $\frac{128 - 8}{128}$ &16$\times$\\
\bottomrule
\end{tabular}
\end{adjustbox}
\vspace{0.5em}
\caption{Storage per 4-weight block and compression savings under structured $2$:$4$ sparsity ~\cite{mishra2021accelerating}, normalized against 32-bit dense baseline (128 bits per 4 weights).}
\label{tab:2to4_savings_fp32}
\end{table}
\begin{table}[h]
\centering
\small
\begin{adjustbox}{max width=0.8\textwidth}
\begin{tabular}{lcccc}
\toprule
\textbf{Bitwidth $(n)$} & \textbf{2:8 Sparse (bits)} & \textbf{Savings vs FP32} & \textbf{Formula} & \textbf{Compression Ratio} \\
\midrule
8-bit & $2 \times 8 + 5 = 21$ & 91.80\% & $\frac{256 - 21}{256}$ & $12.2\times$ \\
4-bit & $2 \times 4 + 5 = 13$ & 94.92\% & $\frac{256 - 13}{256}$ & $19.7\times$ \\
2-bit & $2 \times 2 + 5 = 9$  & 96.48\% & $\frac{256 - 9}{256}$  & $28.4\times$ \\
\bottomrule
\end{tabular}
\end{adjustbox}
\vspace{0.5em}
\caption{Storage per 8-weight block and compression savings under structured 2:8 sparsity, normalized against 32-bit dense baseline (256 bits per 8 weights).}
\label{tab:compression_2to8}
\end{table}



\subsection*{Training Details}
The proposed method is in Algorithm~\ref{alg:training}. The overall process can be summarized as follows: Taking pre-trained model weights as initialization, training the model with $L_{reg}(\mathbf{W},\mathbf{\widehat{W}})$ (Eq.~\eqref{eq:loss}) to reduce the weight discrepancy and updating the weights and other parameters through STE. The scaling factors of $ s_x $ and $ s_w $ are initialized and updated by using the LSQ method. When training the ResNet-18 model with 8$\times$V100, each epoch takes about 6 minutes. It takes about ${MAX\_EPOCH}=120$ epochs to obtain a sparse quantized models.  We follow Nvidia's hyper-parameter settings and training code \footnote{https://github.com/NVIDIA/DeepLearningExamples}. For the DeiT vision transformer, we apply our method to the original training code \cite{wu2019detectron2} and follow its settings. 
\begin{algorithm}

	\caption{Sparse quantization training approach}
	\label{alg:training}
	\begin{algorithmic}[1]
		\WHILE{\textit{epoch $<$ {MAX\_EPOCH}}} \label{lst:line:16}
        \STATE{
		$\mathbf{\hat{x}}=q(\mathbf{x};s_x,b)$ }
		\STATE{
		$\widehat{\mathbf{W}}=q(S(\mathbf{W};N,M);s_w,b)$ } 
	    \STATE{
	    $\mathbf{y}={\widehat{\mathbf{W}}}^\top\mathbf{\hat{x}}$ 
	    }
        \STATE{
            $J(\mathbf{W})= L(\widehat{\mathbf{W}})+\lambda L_{reg}(\mathbf{W},\widehat{\mathbf{W}})$ 
	    }
	    \STATE{
	    Get $\frac{\partial J}{\partial {\mathbf{{W}}}}$ via STE to update $\mathbf{{W}}$, $s_x$, and $s_w$.%
	    }
		\ENDWHILE
	
	\end{algorithmic}
\end{algorithm}

\subsection*{Error Bound for Structured Sparse Quantization}
\begin{theorem}[Structured 2:4 Sparse Quantization Bounds]\label{thm:24-bounds}
Let $\mathbf{w}\in\mathbb{R}^{d}$ and obtain $\hat{\mathbf{w}}$ by structured
$2{:}4$ sparsification, i.e.\ in every 4-element block we keep the two
largest-magnitude entries and set the other two to zero.  
Let the angle between $\mathbf{w}$ and $\hat{\mathbf{w}}$ be
$\theta\!\in[0,\frac{\pi}{2}]$ (so $\cos\theta=\tfrac{\mathbf{w}^{\top}\hat{\mathbf{w}}}
{\|\mathbf{w}\|_{2}\,\|\hat{\mathbf{w}}\|_{2}}$).  
Then
\[
\boxed{\;
  \|\mathbf{w}\|_{2}^{2}\sin^{2}\theta
  \;\;\le\;\;
  \|\mathbf{w}-\hat{\mathbf{w}}\|_{2}^{2}
  \;\;\le\;\;
  2\,\|\mathbf{w}\|_{2}^{2}(1-\cos\theta)
\;}
\tag{A}
\]
and, necessarily,
\[
\boxed{\;
  \cos\theta \;\ge\; \frac{1}{\sqrt{2}}
  \quad\Longleftrightarrow\quad
  \theta\;\le\;45^{\circ}
\;}
\tag{B}
\]
\end{theorem}

\begin{proof}
The proof has two short ingredients.

\paragraph{1.\ 2:4 blocks preserve \(\ge\!\frac12\) of the energy.}
Fix a 4-vector
$\mathbf{z}=(z_1,z_2,z_3,z_4)$ and, w.l.o.g., relabel so that
$|z_1|\!\ge\!|z_2|\!\ge\!|z_3|\!\ge\!|z_4|$.  
Averaging tells us
\[
\frac{z_1^{2}+z_2^{2}}{2}
\;\ge\;
\frac{z_1^{2}+z_2^{2}+z_3^{2}+z_4^{2}}{4}
\;=\;\frac14\|\mathbf{z}\|_{2}^{2},
\]
so
$z_1^{2}+z_2^{2}\ge\tfrac12\|\mathbf{z}\|_{2}^{2}$.  
Applying this block-by-block and summing yields
\[
\boxed{\;
  \|\hat{\mathbf{w}}\|_{2}^{2}\;\ge\;\tfrac12\|\mathbf{w}\|_{2}^{2}
\;}
\quad\Longrightarrow\quad
\|\hat{\mathbf{w}}\|_{2}\;\ge\;\tfrac{1}{\sqrt{2}}\|\mathbf{w}\|_{2}.
\tag{1}
\]

\paragraph{2.\  Standard law-of-cosines decomposition.}
For any pair of vectors with angle $\theta$,
\[
\|\mathbf{w}-\hat{\mathbf{w}}\|_{2}^{2}
=\|\mathbf{w}\|_{2}^{2}+\|\hat{\mathbf{w}}\|_{2}^{2}
-2\|\mathbf{w}\|_{2}\|\hat{\mathbf{w}}\|_{2}\cos\theta
=\bigl(\|\mathbf{w}\|_{2}\sin\theta\bigr)^{2}
+\bigl(\|\hat{\mathbf{w}}\|_{2}-\|\mathbf{w}\|_{2}\cos\theta\bigr)^{2}.
\tag{2}
\]

\paragraph{Lower bound.}
Since the second square in (2) is non-negative,
\[
\|\mathbf{w}-\hat{\mathbf{w}}\|_{2}^{2}\;\ge\;
\|\mathbf{w}\|_{2}^{2}\sin^{2}\theta,
\]
which is exactly the left-hand inequality in (A).  
(The $\tfrac12$ factor that appeared in the draft is unnecessary—it only
weakens the bound.)

\paragraph{Upper bound.}
Use $\|\hat{\mathbf{w}}\|_{2}\le\|\mathbf{w}\|_{2}$ (sparsification never
increases the norm) in the first form of (2):
\[
\|\mathbf{w}-\hat{\mathbf{w}}\|_{2}^{2}
\le
\|\mathbf{w}\|_{2}^{2}+\|\mathbf{w}\|_{2}^{2}
-2\|\mathbf{w}\|_{2}^{2}\cos\theta
=2\,\|\mathbf{w}\|_{2}^{2}(1-\cos\theta),
\]
giving the right-hand inequality in (A).

\paragraph{Angular constraint.}
Combine the Cauchy lower bound $\mathbf{w}^{\top}\hat{\mathbf{w}}
\ge\tfrac12\|\mathbf{w}\|_{2}^{2}$ (sum of preserved squares per block)
with (1):
\[
\cos\theta
=\frac{\mathbf{w}^{\top}\hat{\mathbf{w}}}{\|\mathbf{w}\|_{2}\|\hat{\mathbf{w}}\|_{2}}
\;\ge\;
\frac{\tfrac12\|\mathbf{w}\|_{2}^{2}}
{\|\mathbf{w}\|_{2}\cdot(\|\mathbf{w}\|_{2}/\sqrt2)}
=\frac1{\sqrt2},
\]
establishing (B).  Equality occurs when each 4-block looks like
$(a,a,a,a)$ up to sign, matching the intuitive “worst-case’’ example.

\end{proof}

\subsection*{Tightness}
\begin{proposition}[Both bounds coalesce as \(\theta\to0\)]
For the setting of Theorem~\ref{thm:24-bounds},
denote the quantization error by 
\(E(\theta)=\|\mathbf{w}-\hat{\mathbf{w}}\|^{2}_{2}\)
and recall the bounds
\[
L(\theta):=\|\mathbf{w}\|_{2}^{2}\sin^{2}\theta
\;\le\;
E(\theta)
\;\le\;
U(\theta):=2\|\mathbf{w}\|_{2}^{2}\bigl(1-\cos\theta\bigr).
\]
Then, as \(\theta\to0\),
\[
L(\theta)
=
\|\mathbf{w}\|_{2}^{2}\theta^{2}
\;+\; \mathcal{O}\!\left(\theta^{4}\right),
\qquad
U(\theta)
=
\|\mathbf{w}\|_{2}^{2}\theta^{2}
\;+\; \mathcal{O}\!\left(\theta^{4}\right),
\]
and hence
\[
U(\theta)-L(\theta)\;=\;\mathcal{O}\!\left(\theta^{4}\right).
\]
Consequently, minimising the surrogate loss
\(1-\cos\theta\) (equivalently, \(\tfrac12\theta^{2}+o(\theta^{2})\))
drives the \emph{true} quantisation error
\(E(\theta)\) down quadratically in \(\theta\),
while simultaneously squeezing the gap between the lower
and upper analytical bounds at the even faster
quartic rate \(\theta^{4}\) (Fig.~\ref{fig:lower_upper_bound}).
\begin{figure}[hbpt!]
    \centering
    \includegraphics[width=0.5\linewidth]{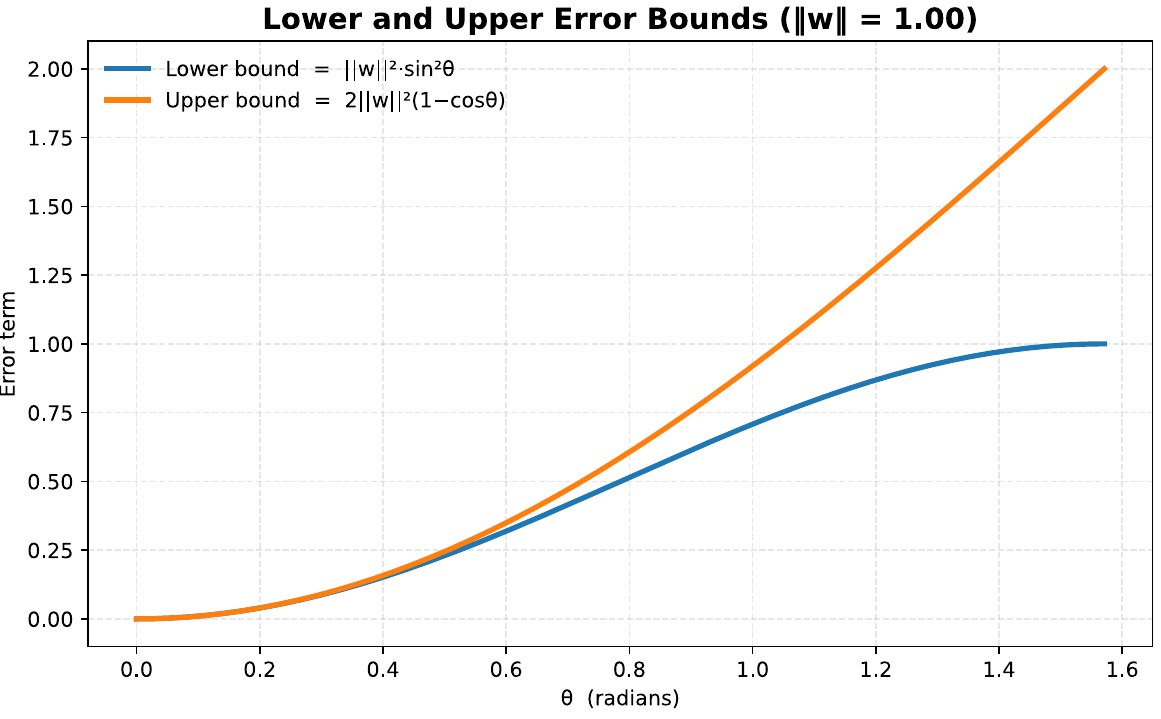}
    \caption{The lower and upper bound trend when minimizing $\theta$.}
    \label{fig:lower_upper_bound}
\end{figure}
\end{proposition}

\begin{proof}
A second-order Maclaurin expansion of the elementary
trigonometric functions yields, for \(\theta\to0\),
\[
\sin\theta
= \theta - \tfrac{\theta^{3}}{6} + \mathcal{O}\!\left(\theta^{5}\right),
\qquad
\cos\theta
= 1 - \tfrac{\theta^{2}}{2} + \tfrac{\theta^{4}}{24}
      + \mathcal{O}\!\left(\theta^{6}\right).
\]

\paragraph{Lower bound.}
\[
L(\theta)
=\|\mathbf{w}\|_{2}^{2}\sin^{2}\theta
=\|\mathbf{w}\|_{2}^{2}
  \bigl(\theta - \tfrac{\theta^{3}}{6}
        + \mathcal{O}(\theta^{5})\bigr)^{2}
=\|\mathbf{w}\|_{2}^{2}\theta^{2}
  +\mathcal{O}\!\left(\theta^{4}\right).
\]

\paragraph{Upper bound.}
\[
U(\theta)
=2\|\mathbf{w}\|_{2}^{2}\bigl(1-\cos\theta\bigr)
=2\|\mathbf{w}\|_{2}^{2}
  \bigl(\tfrac{\theta^{2}}{2}-\tfrac{\theta^{4}}{24}
        +\mathcal{O}(\theta^{6})\bigr)
=\|\mathbf{w}\|_{2}^{2}\theta^{2}
  +\mathcal{O}\!\left(\theta^{4}\right).
\]

\paragraph{Gap between bounds.}
Subtracting the two expansions shows
\(U(\theta)-L(\theta)=\mathcal{O}(\theta^{4})\).
Thus both analytical bounds converge to the same
leading-order term \(\|\mathbf{w}\|_{2}^{2}\theta^{2}\),
while their separation shrinks \emph{two orders faster}
than the error itself .

\end{proof}

\end{document}